\documentclass{article}
\PassOptionsToPackage{numbers,sort&compress}{natbib}

\usepackage{natbib}
\usepackage[
  letterpaper,
  textheight=9in,
  textwidth=5.5in,
  top=1in
  ]{geometry}

\usepackage[utf8]{inputenc} 
\usepackage[T1]{fontenc}    
\usepackage[colorlinks]{hyperref}       
\usepackage{url}            
\usepackage{booktabs}       
\usepackage{amsfonts} 
\usepackage{nicefrac}       
\usepackage{microtype}      

\usepackage{color}
\definecolor{darkblue}{rgb}{0.0,0.0,0.55}
\hypersetup{
  colorlinks = true,
  citecolor  = darkblue,
  linkcolor  = darkblue,
  citecolor  = darkblue,
  filecolor  = darkblue,
  urlcolor   = darkblue,
}

\usepackage{amsmath} 
\usepackage{float}
\usepackage{algorithm}
\usepackage{algorithmic}
\usepackage{amssymb}
\usepackage{amsthm}
\usepackage{xspace}
\usepackage{enumerate}

\usepackage{subcaption}

\usepackage{graphicx}

\newlength\figWsynthleft
\newlength\figWsynthright
\newlength\figWgenesleft
\newlength\figWgenesright
\newlength\figH
\newlength\figW
\newcommand{\algoa}{\textsc{KrK-Picard}\xspace}
\newcommand{\algob}{\textsc{Joint-Picard}\xspace}
\newcommand{\picard}{\textsc{Picard}\xspace}

\newcommand{\Yc}{\mathcal{Y}}
\newcommand{\Pc}{\mathcal{P}}
\newtheorem{theorem}{Theorem}[section]

\newtheorem{prop}[theorem]{Proposition}
\newtheorem{corr}[theorem]{Corollary}

\theoremstyle{definition}
\newtheorem{defn}[theorem]{Definition}

\newcommand{\kron}{\otimes}

\DeclareMathOperator{\vect}{vec}
\DeclareMathOperator{\mat}{mat}
\DeclareMathOperator{\Tr}{Tr}

\newcommand{\krondpp}{\textsc{KronDpp}\xspace}
\newcommand{\krondpps}{\textsc{KronDpps}\xspace}
\newcommand{\reals}{\mathbb{R}}

\title{Kronecker Determinantal Point Processes}

\author{
\normalsize  Zelda Mariet \\
\normalsize  Massachusetts Institute of Technology\\
\normalsize  Cambridge, MA 02139 \\
\normalsize  \texttt{zelda@csail.mit.edu} \\
  \and
\normalsize  Suvrit Sra \\
\normalsize  Massachusetts Institute of Technology \\
\normalsize  Cambridge, MA 02139 \\
\normalsize  \texttt{suvrit@mit.edu} \\
}

\date{}
\begin{document}
\maketitle

\begin{abstract}
  Determinantal Point Processes (DPPs) are probabilistic models over all subsets a ground set of $N$ items. They have recently gained prominence in several applications that rely on ``diverse'' subsets. However, their applicability to large problems is still limited due to the $\mathcal O(N^3)$ complexity of core tasks such as sampling and learning. We enable efficient sampling and learning for DPPs by introducing \krondpp, a DPP model whose kernel matrix decomposes as a tensor product of multiple smaller kernel matrices. This decomposition immediately enables fast \emph{exact} sampling. But contrary to what one may expect, leveraging the Kronecker product structure for speeding up DPP learning turns out to be more difficult. We overcome this challenge, and derive batch and stochastic optimization algorithms for efficiently learning the parameters of a \krondpp. 
\end{abstract}

\section{Introduction}
\vspace*{-6pt}
Determinantal Point Processes (DPPs) are discrete probability models over the subsets of a ground set of $N$ items. They provide an elegant model to assign probabilities to an exponentially large sample, while permitting tractable (polynomial time) sampling and marginalization. They are often used to provide models that balance ``diversity'' and quality, characteristics valuable to numerous problems in machine learning and related areas~\citep{kuleszaBook}. 

The antecedents of DPPs lie in statistical mechanics~\citep{macchi75}, 
but since the seminal work of~\citep{kuleszaThesis} they have made inroads into machine learning. By now they have been applied to a variety of problems such as document and video summarization~\citep{lin12,chao15}, sensor placement~\citep{krause08}, recommender systems~\citep{zhou10}, and object retrieval~\citep{affandi14}. More recently, they have been used to compress fully-connected layers in neural networks~\citep{marietSra15b} and to provide optimal sampling procedures for the Nystr\"om method~\citep{li16}. The more general study of DPP properties has also garnered a significant amount of interest, see e.g.,~\citep{hough06,kulesza11,kuleszaBook,affandi13b,decreuse,lavancier2015,borodin2009,lyons03}. 

However, despite their elegance and tractability, widespread adoption of DPPs is impeded by the $O(N^3)$ cost of basic tasks such as (exact) sampling~\citep{hough06,kuleszaBook} and learning~\citep{hough06,kuleszaBook,gillenwater14,marietSra15}. This cost has motivated a string of recent works on approximate sampling methods such as MCMC samplers~\citep{kang13,li16} or core-set based samplers~\citep{li15}. The task of learning a DPP from data has received less attention; the methods of~\citep{gillenwater14,marietSra15} cost $O(N^3)$ per iteration, which is clearly unacceptable for realistic settings. This burden is partially ameliorated in~\citep{gartrell16}, who restrict to learning low-rank DPPs, though at the expense of being unable to sample subsets larger than the chosen rank.

These considerations motivate us to introduce \krondpp, a DPP model that uses Kronecker (tensor) product kernels. As a result, \krondpp enables us to learn large sized DPP kernels, while also permitting efficient (exact and approximate) sampling. The use of Kronecker products to scale matrix models is a popular and effective idea in several machine-learning settings~\citep{wu05,martens15,flaxman2015fast,zhang15}. But as we will see, its efficient execution for DPPs turns out to be surprisingly challenging. 

To make our discussion more concrete, we recall some basic facts now. Suppose we have a ground set of $N$ items $\Yc = \{1, \ldots, N\}$. A discrete DPP over $\Yc$ is a probability measure $\Pc$ on $2^\Yc$ parametrized by a positive definite matrix $K$ (the \emph{marginal kernel}) such that $0 \preceq K \preceq I$, so that for any $Y \in \Yc$ drawn from $\Pc$, the measure satisfies
\begin{equation}
  \label{eq:1}
  \forall A \subseteq \Yc, \qquad \Pc(A \subseteq Y) = \det(K_A),
\end{equation}
where $K_A$ is the submatrix of $K$ indexed by elements in $A$ (i.e., $K_A = [K_{ij}]_{i,j \in A}$).
If a DPP with marginal kernel $K$ assigns nonzero probability to the empty set, the DPP can alternatively be parametrized by a positive definite matrix $L$ (the \emph{DPP kernel}) so that
\begin{equation}
  \label{eq:2}
  \Pc(Y) \propto \det(L_Y)\quad\implies\quad \Pc(Y) = \frac{\det(L_Y)}{\det(L+I)}.
\end{equation}
A brief manipulation (see e.g.,~\citep[Eq.~15]{kuleszaBook}) shows that when the inverse exists, $L = K(I-K)^{-1}$. The determinants, such as in the normalization constant in~\eqref{eq:2}, make operations over DPPs typically cost  $\mathcal O(N^3)$, which is a key impediment to their scalability.

Therefore, if we consider a class of DPP kernels whose structure makes it easy to compute determinants, we should be able to scale up DPPs. An alternative approach towards scalability is to restrict the size of the subsets, as done in $k$-DPP~\citep{kulesza11} or when using rank-$k$ DPP kernels~\citep{gartrell16} (where $k\ll N$). Both these approaches still require $O(N^3)$ preprocessing for exact sampling; another caveat is that they limit the DPP model by assigning zero probabilities to sets of cardinality greater than $k$. 

In contrast, \krondpp uses a kernel matrix of the form $L = L_1 \kron \ldots \kron L_m$, where each \emph{sub-kernel} $L_i$ is a smaller positive definite matrix. This decomposition has two key advantages: (i) it significantly lowers the number of parameters required to specify the DPP from $N^2$ to $\mathcal O(N^{2/m})$ (assuming the sub-kernels are roughly the same size); and (ii) it enables fast sampling and learning.

For ease of exposition, we describe specific details of \krondpp for $m=2$; as will become clear from the analysis, typically the special cases $m=2$ and $m=3$ should suffice to obtain low-complexity sampling and learning algorithms.

\paragraph{Contributions.} Our main contribution is the \krondpp model along with efficient algorithms for sampling from it and learning a Kronecker factored kernel. 
Specifically, inspired by the algorithm of~\citep{marietSra15}, we develop \algoa (\textbf{Kr}onecker-\textbf{K}ernel Picard), a block-coordinate ascent procedure that generates a sequence of Kronecker factored estimates of the DPP kernel while ensuring monotonic progress on its (difficult, nonconvex) objective function. More importantly, we show how to implement \algoa to run in $\mathcal O(N^2)$ time when implemented as a batch method, and in $\mathcal O(N^{3/2})$ time and $\mathcal O(N)$ space, when implemented as a stochastic method. 
As alluded to above, unlike many other uses of Kronecker models, \krondpp does not admit trivial scaling up, largely due to extensive dependence of DPPs on arbitrary submatrices of the DPP kernel. An interesting theoretical nugget that arises from our analysis is the combinatorial problem that we call \emph{subset clustering}, a problem whose (even approximate) solution can lead to further speedups of our algorithms.

\section{Preliminaries}
We begin by recalling basic properties of Kronecker products needed in our analysis; we omit proofs of these well-known results for brevity. The Kronecker (tensor)  product of $A\in \reals^{p\times q}$ with $B\in \reals^{r\times s}$ two matrices is defined as the $pr\times qs$ block matrix $A\kron B = [a_{ij}B]_{i,j=1}^{p,q}$. 

We denote the block $a_{ij}B$ in $A\kron B$ by $(A\kron B)_{(ij)}$ for any valid pair $(i,j)$, and extend the notation to non-Kronecker product matrices to indicate the submatrix of size $r\times s$ at position $(i,j)$. 
\begin{prop}
  \label{prop:basic}
  Let $A,B,C,D$ be matrices of sizes so that $AC$ and $BD$ are well-defined. Then,
  \begin{enumerate}[(i)]
    \setlength\itemsep{0pt}
    \item If $A, B \succeq 0$, then, $A \kron B \succeq 0$;
    \item If $A$ and $B$ are invertible then so is $A\kron B$, with $(A \kron B)^{-1} = A^{-1} \kron B^{-1}$;
    \item $(A \kron B)(C \kron D)$ = $(AC) \kron (BD)$.
  \end{enumerate}
\end{prop}
An important consequence of Prop.~\ref{prop:basic}$(iii)$ is the following corollary.
\begin{corr}
  \label{corr:eigendecompose}
  Let $A=P_AD_AP_A^\top$ and $B=P_BD_BP_B^\top$ be the eigenvector decompositions of $A$ and $B$. Then, $A\kron B$ diagonalizes as $(P_A \kron P_B)(D_A \kron D_B) (P_A \kron P_B)^\top$.
\end{corr}

We will also need the notion of partial trace operators, which are perhaps less well-known:
\begin{defn}
  Let $A \in \mathbb R^{N_1N_2 \times N_1N_2}$. The \emph{partial traces} $\Tr_1(A)$ and $\Tr_2(A)$ are defined as follows:
    \[\Tr_1(A) := \left[\Tr(A_{(ij)}\right]_{1 \leq i, j \leq N_1} \in \mathbb R^{N_1\times N_1}, \qquad \Tr_2(A) := \sum\nolimits_{i=1}^{N_1} A_{(ii)} \in \mathbb R^{N_2 \times N_2}.\]
\end{defn}
The action of partial traces is easy to visualize: indeed, $\Tr_1(A \kron B) = \Tr(B) A$ and $\Tr_2(A \kron B) = \Tr(A) B$. For us, the most important property of partial trace operators is their positivity.
\begin{prop}
  \label{prop:posdef-operator}
  $\Tr_1$ and $\Tr_2$ are positive operators, i.e., for $A \succ 0$, $\Tr_1(A) \succ 0$ and $\Tr_2(A) \succ 0$.
\end{prop}
\begin{proof}
  Please refer to~\citep[Chap. 4]{bhatia07}.
\end{proof}

\section{Learning the kernel matrix for \krondpp}
\label{sec:learning}
In this section, we consider the key difficult task for \krondpp{}s: learning a Kronecker product kernel matrix from $n$ observed subsets $Y_1,\ldots,Y_n$. Using the definition~(\ref{eq:2}) of $\Pc(Y_i)$, maximum-likelihood learning of a DPP with kernel $L$ results in the optimization problem:
\begin{equation}
  \label{eq:10}
  \arg\max_{L \succ 0}\quad\phi(L), \qquad \phi(L) = \frac 1 n \sum_{i=1}^n \left(\log \det (L_{Y_i}) -\log\det(L+I)\right). 
\end{equation}
This problem is nonconvex and conjectured to be NP-hard~\citep[Conjecture 4.1]{kuleszaThesis}. Moreover the constraint $L\succ 0$ is nontrivial to handle. Writing $U_i$ as the indicator matrix for $Y_i$ of size $N \times |Y_i|$ so that $L_{Y_i} = U_i^\top L U_i$, the gradient of $\phi$ is easily seen to be
\begin{equation}
  \label{eq:delta}
  \Delta := \nabla \phi(L) = \frac 1 n \sum\nolimits_{i=1}^n U_i L_{Y_i}^{-1} U_i^\top - (L+I)^{-1}.
\end{equation} 
In~\citep{marietSra15}, the authors derived an iterative method (``the Picard iteration'') for computing an $L$ that solves $\Delta=0$ by running the simple iteration
\begin{equation}
  \label{eq:picard}
  L \leftarrow L + L\Delta L.
\end{equation}
Moreover, iteration~\eqref{eq:picard} is guaranteed to monotonically increase the log-likelihood $\phi$~\citep{marietSra15}. But these benefits accrue at a cost of $O(N^3)$ per iteration, and furthermore a direct application of~\eqref{eq:picard} cannot guarantee the Kronecker structure required by \krondpp.

\subsection{Optimization algorithm}
Our aim is to obtain an efficient algorithm to (locally) optimize~\eqref{eq:10}. Beyond its nonconvexity, the Kronecker structure $L=L_1\kron L_2$ imposes another constraint. As in~\citep{marietSra15} we first rewrite $\phi$ as a function of $S=L^{-1}$, and re-arrange terms to write it as
\begin{equation}
  \label{eq:9}
  \phi(S) = \underbrace{\vphantom{\sum\nolimits_a^b\frac 1n}\log\det(S)}_{f(S)} + \underbrace{\frac 1 n \sum\nolimits_{i=1}^n \log \det \left(U_i^\top S^{-1} U_i\right) - \log \det (I+S)}_{g(S)}.
\end{equation}
It is easy to see that $f$ is concave, while a short argument shows that $g$ is convex~\citep{marietSra15}. An appeal to the convex-concave procedure~\citep{yuille02} then shows that updating $S$ by solving $\nabla\!f(S^{(k+1)}) +\nabla\!g(S^{(k)})=0$, which is what ~\eqref{eq:picard} does~\citep[Thm. 2.2]{marietSra15}, is guaranteed to monotonically increase $\phi$. 

But for \krondpp this idea does not apply so easily: due the constraint $L = L_1 \kron L_2$ the function
\begin{equation*}
    g_{\tiny\kron} : (S_1,S_2) \rightarrow \tfrac 1 n \sum\nolimits_{i=1}^n \log \det \left(U_i^\top(S_1\kron S_2)^{-1} U_i\right) - \log \det (I+S_1 \kron S_2),
\end{equation*}
fails to be convex, precluding an easy generalization. Nevertheless, for fixed $S_1$ or $S_2$ the functions
\begin{equation*}
  \begin{cases}
    f_1 : S_1 \mapsto f(S_1 \kron S_2) \\
    g_1 : S_1 \mapsto g(S_1 \kron S_2) \\
  \end{cases}, \qquad
  \begin{cases}
    f_2 : S_2 \rightarrow f(S_1 \kron S_2) \\
    g_2 : S_2 \rightarrow g(S_1 \kron S_2) \\
  \end{cases}
\end{equation*}
are once again concave or convex. Indeed, the map $\kron: S_1 \rightarrow S_1 \kron S_2$ is linear and $f$ is concave, and $f_1 = f \circ \kron$ is also concave; similarly, $f_2$ is seen to be concave and $g_1$ and $g_2$ are convex. Hence, by generalizing the arguments of~\citep[Thm. 2]{yuille02} to our ``block-coordinate'' setting, updating via
\begin{equation}
  \label{eq:8}
  \nabla\!f_i\bigl({S_i}^{(k+1)}\bigr) = - \nabla\!g_i\bigl({S_i}^{(k)}\bigr),\quad\text{for}\ i=1,2,
\end{equation}
should increase the log-likelihood $\phi$ at each iteration. We prove below that this is indeed the case, and that updating as per~\eqref{eq:8} ensure positive definiteness of the iterates as well as monotonic ascent.

\subsubsection{Positive definite iterates and ascent} 
In order to show the positive definiteness of the solutions to \eqref{eq:8}, we first derive their closed form. 
\begin{prop}[Positive definite iterates]
  \label{prop:differenciation}
  For $S_1 \succ 0$, $S_2 \succ 0$, the solutions to \eqref{eq:8} are given by the following expressions:
  \begin{align*}
    \nabla\!f_1(X) = - \nabla\! g_1(S_1) &\iff X^{-1} = \Tr_1\!\left((I \kron S_2)(L + L \Delta L)\right)/N_2 \\
    \nabla\!f_2(X) = - \nabla\! g_2(S_2) &\iff X^{-1} = \Tr_2\left((S_1 \kron I) (L + L\Delta L)\right)/N_1.
  \end{align*}
  Moreover, these solutions are positive definite.
\end{prop}
\begin{proof}
  The details are somewhat technical, and are hence given in Appendix~\ref{app:cccp-psd}. We know that $L \succ 0 \implies L + L \Delta L \geq 0$, because $L - L( I+L)^{-1}L \succ 0$. Since the partial trace operators are positive (Prop.~\ref{prop:posdef-operator}), it follows that the solutions to ~\eqref{eq:8} are also positive definite.  
\end{proof}

We are now ready to establish that these updates ensure monotonic ascent in the log-likelihood.
\begin{theorem}[Ascent]
  \label{thm:cccp}
  Starting with $L_1^{(0)} \succ 0$, $L_2^{(0)} \succ 0$, updating according to ~\eqref{eq:8} generates positive definite iterates $L_1^{(k)}$ and $L_2^{(k)}$, and the sequence $\bigl\{\phi\bigl(L_1^{(k)} \kron L_2^{(k)}\bigr)\bigr\}_{k \geq 0}$ is non-decreasing.
\end{theorem}
\begin{proof}
  Updating according to ~\eqref{eq:8} generates positive definite matrices $S_i$, and hence positive definite subkernels $L_i=S_i$. Moreover, due to the convexity of $g_1$ and concavity of $f_1$, for matrices $A,B \succ 0$
  \begin{equation*}
    \begin{aligned}
      f_1(B) \leq f_1(A) + \nabla\!f_1(A)^\top(B-A), \\
      g_1(A) \geq g_1(B) + \nabla\!g_1(B)^\top(A-B).      
    \end{aligned}
  \end{equation*}
  Hence, $f_1(A) + g_1(A) \geq f_1(B) + g_1(B) + (\nabla\!f_1(A) + \nabla\!g_1(B))^\top(A-B)$. 

Thus, if $S_1^{(k)}, S_1^{(k+1)}$ verify \eqref{eq:8}, by setting $A = S_1^{(k+1)}$ and $B=S_1^{(k)}$ we have \[\phi\bigl(L_1^{(k+1)} \kron L_2^{(k)}\bigr) = f_1\!\bigl(S_1^{(k+1)}\bigr) + g_1\!\bigl(S_1^{(k+1)}\bigr) \geq f_1\!\bigl(S_1^{(k)}\bigr) + g_1\!\bigl(S_1^{(k)}\bigr) = \phi\bigl(L_1^{(k)} \kron L_2^{(k)}\bigr).\]
The same reasoning holds for $L_2$, which proves the theorem.
\end{proof}

As $\Tr_1\!\left((I \kron S_2)L\right) = N_2 L_1$ (and similarly for $L_2$), updating as in \eqref{eq:8} is equivalent to updating
\begin{equation*}
    L_1 \leftarrow L_1 + \Tr_1\!\left((I \kron L^{-1}_2)(L \Delta L)\right)/N_2,\qquad
    L_2 \leftarrow L_2 + \Tr_2\left((L^{-1}_1 \kron I) (L\Delta L)\right)/N_1.
\end{equation*}

\textbf{Genearlization.}
We can generalize the updates to take an additional step-size parameter $a$:
\begin{equation*}
  \label{eq:12}
    L_1 \leftarrow L_1 + a \Tr_1\!\left((I \kron L^{-1}_2)(L \Delta L)\right)/N_2,\qquad
    L_2 \leftarrow L_2 + a \Tr_2\left((L^{-1}_1 \kron I) (L\Delta L)\right)/N_1.
\end{equation*}
Experimentally, $a>1$ (as long as the updates remain positive definite) can provide faster convergence, although the monotonicity of the log-likelihood is no longer guaranteed. We found experimentally that the range of admissible $a$ is larger than for Picard, but decreases as $N$ grows larger.

The arguments above easily generalize to the multiblock case. Thus, when learning $L=L_1\kron \cdots \kron L_m$, by writing $E_{ij}$ the matrix with a 1 in position $(i,j)$ and zeros elsewhere, we update $L_k$ as \[(L_k)_{ij} \leftarrow (L_k)_{ij} + N_k/(N_1 \ldots N_m)\Tr\left[(L_1 \kron \ldots \kron L_{k-1} \kron E_{ij} \kron L_{k+1} \kron \ldots \kron L_m )(L \Delta L)\right].\]

From the above updates it is not transparent whether the Kronecker product saves us any computation. In particular, it is not clear whether the updates can be implemented to run faster than $O(N^3)$. We show below in the next section how to implement these updates efficiently.

\subsubsection{Algorithm and complexity analysis}
From Theorem~\ref{thm:cccp}, we obtain Algorithm~\ref{algo:cccp} (which is different from the Picard iteration in~\citep{marietSra15}, because it operates alternatingly on each subkernel). It is important to note that a further speedup to Algorithm~\ref{algo:cccp} can be obtained by performing stochastic updates, i.e., instead of computing the full gradient of the log-likelihood, we perform our updates using only one (or a small minibatch) subset $Y_i$ at each step instead of iterating over the entire training set; this uses the stochastic gradient $\Delta = U_i L_{Y_i}^{-1} U_i^\top - (I+L)^{-1}$. 
\begin{algorithm}[ht]\small
   \caption{\algoa iteration}
   \label{algo:cccp}
\begin{algorithmic}
   \STATE {\bfseries Input:} Matrices $L_1, L_2$, training set $T$, parameter $a$.
   \FOR{$i=1$ {\bfseries to} maxIter}
   \STATE $L_1 \leftarrow L_1 + a \Tr_1\!\left((I \kron L^{-1}_2)(L \Delta L)\right)/N_2$ \qquad// or update stochastically
   \STATE $L_2 \leftarrow L_2 + a \Tr_2\left((L^{-1}_1 \kron I) (L\Delta L)\right)/N_1$\qquad// or update stochastically
   \ENDFOR

   \textbf{return} $(L_1,L_2)$
\end{algorithmic}
\end{algorithm}
The crucial strength of Algorithm~\ref{algo:cccp} lies in the following result:

\begin{theorem}[Complexity]
  \label{thm:complexity}
  For $N_1 \approx N_2 \approx \sqrt N$, the updates in Algorithm~\ref{algo:cccp} can be computed in $\mathcal O(n\kappa^3 + N^2)$ time and $\mathcal O(N^2)$ space, where $\kappa$ is the size of the largest training subset. Furthermore, stochastic updates can be computed in $\mathcal O(N\kappa^2 + N^{3/2})$ time and $\mathcal O(N + \kappa^2)$ space.
\end{theorem}
Indeed, by leveraging the properties of the Kronecker product, the updates can be obtained without computing $L\Delta L$. This result is non-trivial: the components of $\Delta$, $\frac 1 n \sum_i U_i L_{Y_i}^{-1}U_i^\top$ and $(I+L)^{-1}$, must be considered separately for computational efficiency. The proof is provided in App.~\ref{app:cccp-complexity}. However, it seems that considering more than 2 subkernels does not lead to further speed-ups.

If $N_1 \approx N_2 \approx \sqrt N$, these complexities become:
\begin{itemize}
  \setlength\itemsep{1pt}
  \item for non-stochastic updates: $\mathcal O(n\kappa^3 + N^2)$ time, $\mathcal O(N^2)$ space,
  \item for stochastic updates: $\mathcal O(N\kappa^3 + N^{3/2})$ time, $\mathcal O(\kappa^2 + N)$ space.
\end{itemize}

This is a marked improvement over~\citep{marietSra15}, which runs in $\mathcal O(N^2)$ space and $\mathcal O(n\kappa^3 + N^3)$ time (non-stochastic) or $\mathcal O(N^3)$ time (stochastic); Algorithm~\ref{algo:cccp} also provides faster stochastic updates than~\citep{gartrell16}. However, one may wonder if by learning the sub-kernels by alternating updates the log-likelihood converges to a sub-optimal limit. The next section discusses how to jointly update $L_1$ and $L_2$.

\subsection{Joint updates}
We also analyzed the possibility of updating $L_1$ and $L_2$ jointly: we update $L \leftarrow L + L \Delta L$ and then recover the Kronecker structure of the kernel by defining the updates $L_1'$ and $L_2'$ such that:
\begin{equation}
  \label{eq:4}
  \begin{cases}
    (L_1',L_2') \text{ minimizes } \|L + L \Delta L - L_1' \kron L_2'\|^2_F \\
    L_1' \succ 0, L_2' \succ 0, \|L_1'\| = \|L_2'\|
  \end{cases}
\end{equation}
We show in appendix~\ref{app:joint-updates} that such solutions exist and can be computed by from the first singular value and vectors of the matrix $R = \left[\vect((L^{-1} + \Delta)_{(ij)})^\top\right]_{i,j=1}^{N_1}$. Note however that in this case, there is no guaranteed increase in log-likelihood. The pseudocode for the related algorithm (\algob) is given in appendix~\ref{app:joint-pseudocode}. An analysis similar to the proof of Thm.~\ref{thm:complexity} shows that the updates can be obtained $\mathcal O(n \kappa^3 + \max(N_1,N_2)^4)$.

\subsection{Memory-time trade-off}
Although \krondpps have tractable learning algorithms, the memory requirements remain high for non-stochastic updates, as the matrix $\Theta = \frac 1 n \sum_i U_i L^{-1}_{Y_i} U_i^\top$ needs to be stored, requiring $\mathcal O(N^2)$ memory. However, if the training set can be subdivised such that 
\begin{equation}
  \label{eq:6}
\{Y_1, \ldots, Y_n\} = \cup_{k=1}^m S_k \quad s.t. \quad \forall k, \left\vert \cup_{Y \in S_k} Y \right\vert < z,
\end{equation}
$\Theta$ can be decomposed as $\frac 1 n \sum_{k=1}^m \Theta_k$ with $\Theta_k = \sum_{Y_i \in S_k} U_i L_{Y_i}^{-1} U_i^\top$. Due to the bound in Eq.~\ref{eq:6}, each $\Theta_k$ will be sparse, with only $z^2$ non-zero coefficients. We can then store each $\Theta_k$ with minimal storage and update $L_1$ and $L_2$ in $\mathcal O(n\kappa^3 + mz^2 + N^{3/2})$ time and $\mathcal O(mz^2 + N)$ space. 

Determining the existence of such a partition of size $m$ is a variant of the NP-Hard Subset-Union Knapsack Problem (SUKP)~\citep{goldschmidt94} with $m$ knapsacks and where the value of each item (i.e. each $Y_i$) is equal to 1: a solution to SUKP of value $n$ with $m$ knapsacks is equivalent to a solution to~\ref{eq:6}. However, an approximate partition can also be simply constructed via a greedy algorithm.

\section{Sampling}
\label{sec:sampling}
\vspace*{-4pt}
Sampling exactly (see Alg.~\ref{algo:sampling} and~\citep{kuleszaBook}) from a full DPP kernel costs $\mathcal O(N^3 + Nk^3)$ where $k$ is the size of the sampled subset. The bulk of the computation lies in the initial eigendecomposition of $L$; the $k$ orthonormalizations cost $\mathcal O(Nk^3)$. Although the eigendecomposition need only happen once for many iterations of sampling, exact sampling is nonetheless intractable in practice for large $N$.

\begin{algorithm}[ht]\small
   \caption{Sampling from a DPP kernel $L$}
   \label{algo:sampling}
\begin{algorithmic}
   \STATE {\bfseries Input:} Matrix $L$.
   \STATE Eigendecompose $L$ as $\{(\lambda_i,v_i)\}_{1 \leq i \leq N}$.
   \STATE $J \leftarrow \emptyset$
   \FOR{$i=1$ {\bfseries to} $N$}
   \STATE $J \rightarrow J \cup \{i\}$ with probability $\lambda_i / (\lambda_i + 1)$.
   \ENDFOR
   \STATE $V \leftarrow \{v_i\}_{i \in J}$, $Y \leftarrow \emptyset$
   \WHILE{$\vert V \vert> 0$}
   \STATE Sample $i$ from $\{1 \ldots N\}$ with probability $\frac 1{\vert V\vert} \sum_{v \in V} v_i^2$
   \STATE $Y \leftarrow Y \cup \{i\}$,  $V \leftarrow V_\bot$, where $V_\bot$ is an orthonormal basis of the subspace of $V$ orthonormal to $e_i$
   \ENDWHILE
   
\textbf{return} $Y$
\end{algorithmic}
\end{algorithm}

It follows from Prop.~\ref{corr:eigendecompose} that for \krondpps, the eigenvalues $\lambda_i$ can be obtained in $\mathcal O(N_1^3 + N_2^3)$, and the $k$ eigenvectors in $\mathcal O(kN)$ operations. For $N_1 \approx N_2 \approx \sqrt N$,  exact sampling thus only costs $\mathcal O(N^{3/2}+Nk^3)$. If $L = L_1 \kron L_2 \kron L_3$, the same reasoning shows that exact sampling becomes linear in $N$, only requiring  $\mathcal O(Nk^3)$ operations. 

One can also resort to MCMC sampling; for instance such a sampler was considered in~\citep{kang13} (though with an incorrect mixing time analysis). The results of~\citep{li16} hold only for $k$-DPPs, but suggest their MCMC sampler may possibly take $O(N^2\log(N/\epsilon))$ time for full DPPs, which is impractical. Nevertheless if one develops faster MCMC samplers, they should also be able to profit from the Kronecker product structure offered by \krondpp.

\vspace*{-8pt}
\section{Experimental results}
\label{sec:experiments}
\vspace*{-4pt}
In order to validate our learning algorithm, we compared \algoa to \algob and to the Picard iteration (\picard) on multiple real and synthetic datasets.\footnote{All experiments were repeated 5 times and averaged, using MATLAB on a Linux Mint system with 16GB of RAM and an i7-4710HQ CPU @ 2.50GHz.} 

\vspace*{-4pt}
\subsection{Synthetic tests}
\vspace*{-4pt}
All three algorithms were used to learn from synthetic data drawn from a ``true'' kernel. The sub-kernels were initialized by $L_i = X^\top X$, with $X$'s coefficients drawn uniformly from $[0,\sqrt 2]$; for \picard, $L$ was initialized with $L_1 \kron L_2$. 
\begin{figure}[!h]
  \centering
  \begin{subfigure}{\linewidth}
    \centering
    \includegraphics{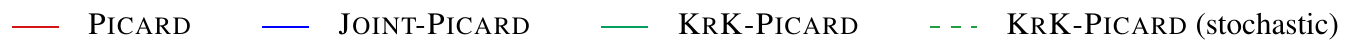}
  \end{subfigure}
  \begin{subfigure}{1.28\figWsynthleft}
    \centering
    \includegraphics{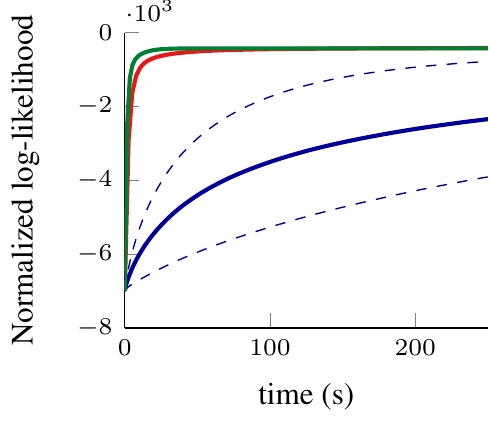}
    \caption{$N_1=N_2=50$}
    \label{fig:synth-small}
  \end{subfigure}
  \begin{subfigure}{1.15\figWsynthright}
    \centering
    \includegraphics{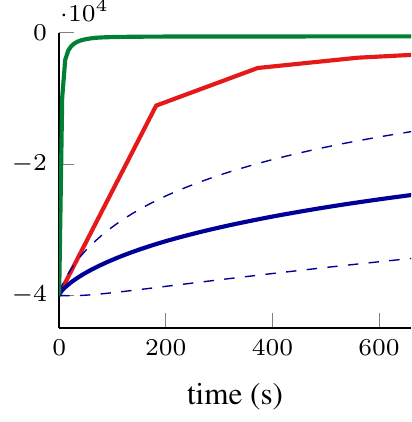}
    \caption{$N_1=N_2=100$}
    \label{fig:synth-large}
  \end{subfigure}
  \begin{subfigure}{1.15\figWsynthright}
    \centering
    \includegraphics{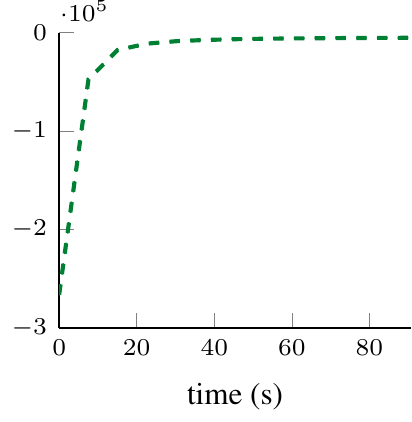}
    \caption{$N_1=100, N_2=500$}
    \label{fig:synth-very-large}
  \end{subfigure}
  \caption{$a=1$; the thin dotted lines indicated the standard deviation from the mean.}
  \label{fig:synthetic}
\end{figure}
For Figures~\ref{fig:synth-small} and~\ref{fig:synth-large}, training data was generated by sampling 100 subsets from the true kernel with sizes uniformly distributed between 10 and 190. 

To evaluate \algoa on matrices too large to fit in memory and with large $\kappa$, we drew samples from a $50\cdot 10^3\!\times\!50\cdot 10^3$ kernel of rank $1,000$ (on average $|Y_i| \approx 1,000$), and learned the kernel stochastically (only \algoa could be run due to the memory requirements of other methods); the likelihood drastically improves in only two steps (Fig.\ref{fig:synth-very-large}).

As shown in Figures~\ref{fig:synth-small} and~\ref{fig:synth-large}, \algoa converges significantly faster than \picard, especially for large values of $N$. However, although \algob also increases the log-likelihood at each iteration, it converges much slower and has a high standard deviation, whereas the standard deviations for \picard and \algoa are barely noticeable. For these reasons, we drop the comparison to \algob in the subsequent experiments.

\subsection{Small-scale real data: baby registries}
We compared \algoa to \picard and EM~\citep{gillenwater14} on the baby registry dataset (described in-depth in~\citep{gillenwater14}), which has also been used to evaluate other DPP learning algorithms~\citep{gillenwater14,marietSra15,gartrell16}. The dataset contains 17 categories of baby-related products obtained from Amazon. We learned kernels for the 6 largest categories ($N=100$); in this case, \picard is sufficiently efficient to be prefered to \algoa; this comparison serves only to evaluate the quality of the final kernel estimates.

The initial marginal kernel $K$ for EM was sampled from a Wishart distribution with $N$ degrees of freedom and an identity covariance matrix, then scaled by $1/N$; for \picard, $L$ was set to $K(I-K)^{-1}$; for \algoa, $L_1$ and $L_2$ were chosen (as in \algob) by minimizing $\|L - L_1 \kron L_2\|$. Convergence was determined when the objective change dipped below a threshold $\delta$. As one EM iteration takes longer than one Picard iteration but increases the likelihood more, we set $\delta_{\textsc{Pic}} = \delta_{\textsc{KrK}} = 10^{-4}$ and $\delta_{\textsc{EM}} = 10^{-5}$.

The final log-likelihoods are shown in Table~\ref{tab:babies}; we set the step-sizes to their largest possible values, i.e. $a_{\textsc{Pic}} = 1.3$ and $a_{\textsc{KrK}} = 1.8$. Table~\ref{tab:babies} shows that \algoa obtains comparable, albeit slightly worse log-likelihoods than \picard and \textsc{EM}, which confirms that for tractable $N$, the better modeling capability of full kernels make them preferable to \krondpps.

\begin{table}[t]
  \caption{Final log-likelihoods for each large category of the baby registries dataset}
    \begin{subtable}{.5\linewidth}
      \centering
      \caption{Training set}
      \small
      \begin{tabular}{lccc}
        \toprule
        Category & \textsc{EM}      & \picard      & \algoa        \\
        \midrule
        apparel & -10.1 & -10.2 & -10.7 \\ 
        bath & -8.6 & -8.8 & -9.1 \\ 
        bedding & -8.7 & -8.8 & -9.3 \\ 
        diaper & -10.5 & -10.7 & -11.1 \\ 
        feeding & -12.1 & -12.1 & -12.5 \\ 
        gear & -9.3 & -9.3 & -9.6 \\ 
        \bottomrule
      \end{tabular}
    \end{subtable}
    \begin{subtable}{.5\linewidth}
      \centering
      \caption{Test set}
      \small
      \begin{tabular}{lccc}
        \toprule
        Category & \textsc{EM}      & \picard      & \algoa        \\
        \midrule
        apparel & -10.1 & -10.2 & -10.7 \\ 
        bath & -8.6 & -8.8 & -9.1 \\ 
        bedding & -8.7 & -8.8 & -9.3 \\ 
        diaper & -10.6 & -10.7 & -11.2 \\ 
        feeding & -12.2 & -12.2 & -12.6 \\ 
        gear & -9.2 & -9.2 & -9.5 \\ 
        \bottomrule
      \end{tabular}
    \end{subtable}
  \label{tab:babies}
\end{table}

\subsection{Large-scale real dataset: GENES}
Finally, to evaluate \algoa on large matrices of real-world data, we train it on data from the GENES~\citep{batmanghelich14} dataset (which has also been used to evaluate DPPs in~\citep{li15,batmanghelich14}). This dataset consists in 10,000 genes, each represented by 331 features corresponding to the distance of a gene to hubs in the BioGRID gene interaction network. 

\begin{table}[b]
  \caption{Average runtime and performance on the GENES dataset for $N_1 = N_2 = 100$}
  \centering \small
  \begin{tabular}{cccc}
    \toprule
    ~ & \picard & \algoa & \algoa (stochastic) \\
    \midrule
    Average runtime & 161.5 $\pm$ 17.7 s & 8.9 $\pm$ 0.2 s& 1.2 $\pm$ 0.02 s\\
    NLL increase (1st iter.) & $(2.81 \pm 0.03) \cdot 10^4$ & $(2.96 \pm 0.02) \cdot 10^4$ & $(3.13 \pm 0.04) \cdot 10^4$ \\
    \bottomrule
  \end{tabular}
  \label{tab:runtimes}
\end{table}

We construct a ground truth Gaussian DPP kernel on the GENES dataset and use it to obtain 100 training samples with sizes uniformly distributed between 50 and 200 items. Similarly to the synthetic experiments, we initialized \algoa's kernel by setting $L_i = X_i^\top X_i$ where $X_i$ is a random matrix of size $N_1 \times N_1$; for \picard, we set the initial kernel $L=L_1 \kron L_2$. 

\begin{figure}[!h]
  \centering
  \begin{subfigure}{\linewidth}
    \centering
    \includegraphics{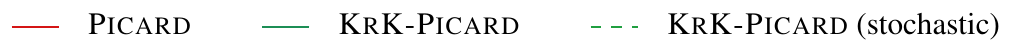}
  \end{subfigure}
  \begin{subfigure}{1.3\figWgenesleft}
    \includegraphics{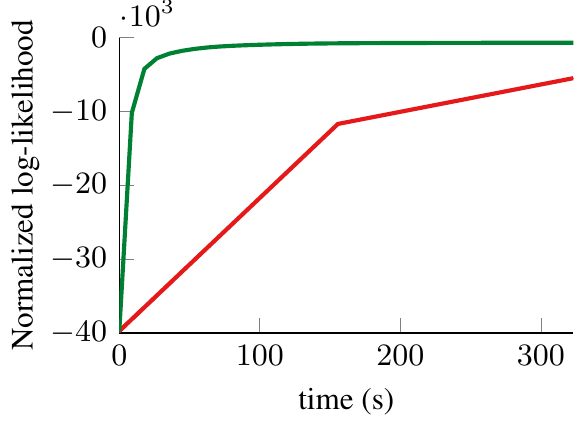}
    \caption{Non-stochastic learning}
    \label{fig:genetic-non-stochastic}
  \end{subfigure}
  \begin{subfigure}{1.3\figWgenesright}
    \includegraphics{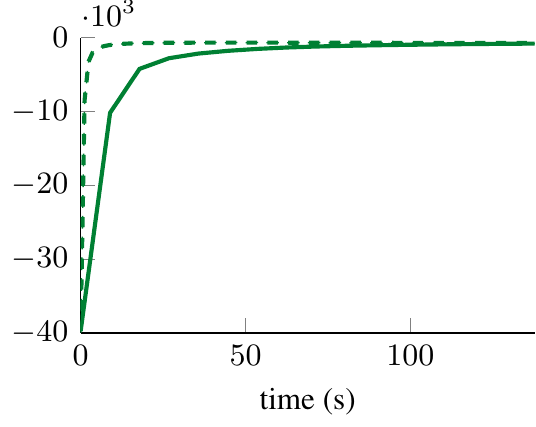}
    \caption{Stochastic vs. non-stochastic}
    \label{fig:genetic-stochastic}
  \end{subfigure}
  \caption{$n=150$, $a=1$.}
  \label{fig:genetic}
\end{figure}

Figure~\ref{fig:genetic} shows the performance of both algorithms. As with the synthetic experiments, \algoa converges much faster; stochastic updates increase its performance even more, as shown in Fig.~\ref{fig:genetic-stochastic}. Average runtimes and speed-up are given in Table~\ref{tab:runtimes}: \algoa runs almost an order of magnitude faster than \picard, and stochastic updates are more than two orders of magnitude faster, while providing slightly larger initial increases to the log-likelihood.

\section{Conclusion and future work}
\vspace*{-4pt}
We introduced \krondpps, a variant of DPPs with kernels structured as the Kronecker product of $m$ smaller matrices, and showed that typical operations over DPPs such as sampling and learning the kernel from data can be made efficient for \krondpps on previously untractable ground set sizes.

By carefully leveraging the properties of the Kronecker product, we derived for $m=2$ a low-complexity algorithm to learn the kernel from data which guarantees positive iterates and a monotonic increase of the log-likelihood, and runs in $\mathcal O(n\kappa^3 + N^2)$ time. This algorithm provides even more significant speed-ups and memory gains in the stochastic case, requiring only $\mathcal O(N^{3/2} + N\kappa^2)$ time and $\mathcal O(N + \kappa^2)$ space. Experiments on synthetic and real data showed that \krondpps can be learned efficiently on sets large enough that $L$ does not fit in memory.

While discussing learning the kernel, we showed that $L_1$ and $L_2$ cannot be updated simultaneously in a CCCP-style iteration since $g$ is not convex over $(S_1,S_2)$. However, it can be shown that $g$ is geodesically convex over the Riemannian manifold of positive definite matrices, which suggests that deriving an iteration which would take advantage of the intrinsic geometry of the problem may be a viable line of future work. 

\krondpps also enable fast sampling, in $\mathcal O(N^{3/2} + Nk^3)$ operations when using two sub-kernels and in $\mathcal O(Nk^3)$ when using three sub-kernels; this allows for exact sampling at comparable or even better costs than previous algorithms for approximate sampling. However, as we improve computational efficiency $L$, the subset size $k$ becomes limiting, due to the $\mathcal O(Nk^3)$ cost of sampling and learning. A necessary line of future work to allow for truly scalable DPPs is thus to overcome this computational bottleneck.

\bibliographystyle{abbrvnat}
\setlength{\bibsep}{1pt}
{\small\vspace*{-5pt}\bibliography{efficient-dpp}}

\clearpage
\appendix

\begin{center}
  \bfseries\large
  Appendix: Kronecker Determinantal Point Processes
\end{center}

\section{Proof of Prop.~\ref{prop:differenciation}}
\label{app:cccp-psd}
We use `$\vect$' to denote the operator that stacks columns of a matrix to form a vector;  conversely, `$\mat$' takes a vector with $k^2$ coefficients and returns a $k \times k$ matrix.

Let $L = L_1 \kron L_2$, $S_1 = L_1^{-1}, S_2 = L_2^{-1}$ and $S = S_1 \kron S_2 = L^{-1}$. We note $E_{ij}$ the matrix with all zeros except for a 1 at position $(i,j)$, its size being clear from context. We wish to solve
  \begin{equation}
    \label{eq:3}
    \nabla\!f_2(X) = -\nabla\!g_2(S_1) \quad \text{ and } \nabla\!f_1(X) = -\nabla\!g_1(S_2)
  \end{equation}
  It follows from the fact that \[\log\det(S_1 \kron S_2) = N_2 \log \det S_1 + N_1 \log \det S_2\] that $\nabla\!f_{S_2}(X) = N_2X^{-1}$ and $\nabla\!f_{S_1}(X) = N_1 X^{-1}$. Moreover, we know that 
  \begin{align*}
    \nabla\!g(S) &= -(I+S)^{-1} - S^{-1}\frac 1 n\sum\nolimits_i U_i(U_i^\top S^{-1} U_i)^{-1}U_i S^{-1}  \\
                 &= - S^{-1} -  S^{-1}\left(\frac 1 n\sum\nolimits_i U_i(U_i^\top S^{-1} U_i)^{-1}U_i - (I+S^{-1})^{-1} \right)S^{-1} \\
                 &= -(L + L \Delta L).
  \end{align*}

  The Jacobian of $S_1 \rightarrow S_1 \kron S_2$ is given by $J = \begin{pmatrix} \vect(E_{11} \kron S_2), \ldots, \vect(E_{N_1N_1} \kron S_2) \end{pmatrix}$.
  Hence,
  \begin{align*}
    \nabla\!f_1(X)_{ij} = -(\nabla\!g_1(S_1))_{ij} &\iff N_2X^{-1}_{ij} = (J^\top \vect(-\nabla g(S)))_{ij} \\
                                                 &\iff N_2X^{-1}_{ij} = \vect(E_{ij} \kron S_2)^\top \vect(L + L \Delta L) \\ 
                                                 &\iff N_2X^{-1}_{ij} = \Tr((E_{ij} \kron S_2) (L + L \Delta L)) \\
                                                 &\iff N_2X^{-1}_{ij} = \Tr( S_2(L + L \Delta L)_{(ij)}) \\
                                                 &\iff N_2X^{-1}_{ij} = \Tr\left(((I \kron S_2) (L + L \Delta L))_{(ij)}\right)
  \end{align*}
  The last equivalence is simply the result of indices manipulation. Thus, we have 
  \[  \nabla f_2(X) = -\nabla g_2(S_1) \iff X^{-1} = \frac 1 {N_2} \Tr_1\!\left((I \kron S_2)(L + L \Delta L)\right)\]

  Similarly, by setting $J'= \left(\vect(S_1 \kron E_{11}), \ldots, \vect(S_1 \kron E_{N_1N_1})\right)$, we have that
  \begin{align*}
    \nabla\!f_2(X)_{ij} = -(\nabla\!g_2(S_2))_{ij} &\iff N_1X^{-1}_{ij} = (J'^\top \vect(-\nabla g(S)))_{ij} \\
                                                   &\iff N_1X^{-1}_{ij} = \vect(S_1 \kron E_{ij})^\top \vect(L + L \Delta L) \\
                                                   &\iff N_1X^{-1}_{ij} = \Tr((S_1 \kron E_{ij})(L + L \Delta L)) \\
                                                   &\iff N_1X^{-1}_{ij} = \left(\sum\nolimits_{k,\ell=1}^{N_1} {S_1}_{k\ell} (L + L\Delta L)_{(\ell k)}\right)_{ij} \\
                                                   &\iff N_1X^{-1}_{ij} = \left(\sum\nolimits_{\ell=1}^{N_1} ((S_1 \kron I)(L + L\Delta L))_{(\ell \ell)}\right)_{ij}
  \end{align*}
  Hence, 
  \[\nabla f_{S_1}(X) = -\nabla g_{S_1}(S_2) \iff X^{-1} = \frac 1 {N_1} \Tr_2\left((S_1 \kron I) (L + L\Delta L)\right),\]
  which proves Prop.~\ref{prop:differenciation}. \qed

\section{Efficient updates for \algoa}
\label{app:cccp-complexity}
The updates to $L_1$ and $L_2$ are obtained efficiently through different methods; hence, the proof to Thm.~\ref{thm:complexity} is split into two sections. We write \[\Theta = \frac 1 n \sum_{i=1}^n U_iL_{Y_i}^{-1}U_i^\top \quad \text { (or }\Theta = U_iL_{Y_i}^{-1}U_i^\top \text{ for stochastic updates)}\] so that $\Delta = \Theta - (I+L)^{-1}$. Recall that $(A\kron B)_{(ij)} = a_{ij}B$.
\subsection{Updating $L_1$}
We wish to compute $X = \Tr_1\!\left((I \kron L_2^{-1})(L \Delta L)\right)$ efficiently. We have
\begin{align*}
  X_{ij} &= \Tr\left[((I \kron L_2^{-1})(L \Delta L))_{(ij)}\right] \\
         &= \Tr\left[L_2^{-1} (L\Delta L)_{(ij)}\right] \\
         &= \Tr\left[L_2^{-1}\sum\nolimits_{k,\ell=1}^{N_1} L_{(ik)} \Delta_{(k\ell)}L_{(\ell j)}\right] \\
         &= \sum\nolimits_{k,\ell=1}^{N_1} {L_1}_{ik} {L_1}_{\ell j} \Tr(L_2^{-1} L_2 \Delta_{(k\ell)} L_2) \\
         &= \sum\nolimits_{k,\ell=1}^{N_1} {L_1}_{ik} {L_1}_{\ell j} \underbrace{\Tr(\Theta_{(k\ell)} L_2)}_{A_{k\ell}} - \underbrace{\Tr((I+L)^{-1}_{(k\ell)} L_2)}_{B_{k\ell}} \\
         &= (L_1 A L_1 - L_1 B L_1)_{ij}.
 \end{align*}

The $N_1 \times N_1$ matrix $A$ can be computed in $\mathcal O(n\kappa^3 + N_1^2N_2^2)$ simply by pre-computing $\Theta$ in $\mathcal O(n\kappa^3)$ and then computing all $N_1^2$ traces in $\mathcal O(N_2^2)$ time. When doing stochastic updates for which $\Theta$ is sparse with only $\kappa^2$ non-zero coefficients, computing $A$ can be done in $\mathcal O(N_1^2\kappa^2 + \kappa^3)$. 

By diagonalizing $L_1 = P_1 D_1 P_1^\top$ and $L_2 = P_2 D_2 P_2^\top$, we have $(I+L)^{-1} = P D P^\top$ with $P = P_1 \kron P_2$ and $D = (I+ D_1 \kron D_2)^{-1}$. $P_1,P_2,D_1,D_2$ and $D$ can all be obtained in $\mathcal O(N_1^3 + N_2^3 + N_1N_2)$ as a consequence of Prop.~\ref{prop:basic}. Then 
\begin{align*}
  B_{ij} &= \Tr((I+L)^{-1}_{(ij)}L_2) \\
         &= \sum_k \Tr(P_{(ik)}D_{(kk)}P^\top_{(kj)}L_2) \\
         &= \sum_k {P_1}_{ik}{P_1}_{jk}\Tr(P_2 D_{(kk)}P_2^\top P_2 D_2 P_2^\top) \\
         &= \sum_k {P_1}_{ik}{P_1}_{jk} \underbrace{\Tr(D_{(kk)}D_2)}_{\alpha_k}.
\end{align*}
Let $\widehat D = \text{diag}(\alpha_1, \ldots, \alpha_{N_1})$, which can be computed in $\mathcal O(N_1N_2)$. Then $L_1BL_1 = P_1 D_1 \widehat D D_1 P_1$ is computable in $\mathcal O(N_1^3 + N_2^3)$.

Overall, the update to $L_1$ can be computed in $\mathcal O(n\kappa^3 + N_1^2N_2^2 + N_1^3+N_2^3)$, or in $\mathcal O(N_1^2\kappa^2 + \kappa^3 + N_1^3 + N_2^3)$ if the updates are stochastic. Moreover, if $\Theta$ is sparse with only $z$ non-zero coefficients (for stochastic updates $z=\kappa$), $A$ can be computed in $\mathcal O(\kappa^2)$ space, leading to an overall $\mathcal O(z^2 + N_1^2 + N_2^2)$ memory cost.

\subsection{Updating $L_2$}
We wish to compute $X = \Tr_2\!\left[(L_1^{-1} \kron I)(L \Delta L)\right]$ efficiently.
\begin{align*}
  X &= \sum\nolimits_{i=1}^{N_1}\left((L_1^{-1}\kron I)(L\Delta L)\right)_{(ii)} \\
     &= \sum\nolimits_{i=1}^{N_1} \left((I \kron L_2) (\Theta - (I+L)^{-1}) (L_1 \kron L_2)\right)_{(ii)} \\
     &= \sum\nolimits_{i,j=1}^{N_1}{L_1}_{ij}L_2 \Theta_{(ij)}L_2 - \sum_{i=1}^{N_1}((I \kron L_2)(I+L)^{-1}(L_1 \kron L_2))_{(ii)}\\
     &= \underbrace{L_2 \sum\nolimits_{i,j=1}^{N_1}{L_1}_{ij}\Theta_{(ij)}L_2}_{A} - \underbrace{\sum\nolimits_{i=1}^{N_1}((I \kron L_2)(I+L)^{-1}(L_1 \kron L_2))_{(ii)}}_B \\
\end{align*}
$A$ can be computed in $\mathcal O(n\kappa^3 + N_1^2N_2^2 + N_2^3)$. As before, when doing stochastic updates $A$ can be computed in $\mathcal O(N_1^2\kappa^2 + \kappa^3 + N_2^3)$ time and $\mathcal O(N_2^2+N_1^2 + \kappa^2)$ space due to the sparsity of $\Theta$. 

Regarding $B$, as all matrices commute, we can write \[(I \kron L_2)(I+L)^{-1}(L_1 \kron L_2) = (P_1 \kron P_2) \Lambda (P_1 \kron P_2)\] where $\Lambda = (I \kron D_2)(I+D_1\kron D_2)^{-1}(D_1 \kron D_2)$ is diagonal and is obtained in $\mathcal O(N_1^3 + N_2^3 + N_1N_2)$. Moreover, 
\[B = \sum\nolimits_{i=1}^{N_1} (P \Lambda P^\top)_{(ii)} = P_2\left(\sum\nolimits_{i,k=1}^{N_1}{P_1}_{ik} \Lambda_{(kk)}{P_1}_{ik}\right) P_2^\top,\]
which allows us to compute $B$ in $\mathcal O(N_1^2N_2 + N_2^3 + N_1^3)$ total.

Overall, we can obtain $X$ in $\mathcal O(n\kappa^3 + N_1^2N_2^2 + N_1^3 + N_2^3)$ or in $\mathcal O(N_1^2\kappa^2 + N_1^2N_2 + N_1^3 + N_2^3)$ for stochastic updates, in which case only $\mathcal O(N_1^2 + N_2^2 + \kappa^2)$ space is necessary.

\section{Proof of validity for joint updates}
\label{app:joint-updates}
In order to minimize the number of matrix multiplications, we equivalently (due to the properties of the Frobenius norm) minimize the equation 
\begin{equation}
  \label{eq:5}
  \|L^{-1} + \Delta - X \kron Y\|^2_F
\end{equation}
and set $
\begin{cases}
L_1' \leftarrow L_1 X L_1\\  
L_2' \leftarrow L_2 Y L_2.
\end{cases}$.

\begin{theorem}
  \label{thm:joint-psd}
  Let $L \succ 0$. Define $R := [ \vect(L_{(11)})^\top;\ldots;\vect(L_{(N_1N_1)})^\top]_{i,j=1}^{N_1} \in \mathbb R^{N_1N_1 \times N_2N_2}$. 

  Suppose that $R$ has an eigengap between its largest singular value and the next, and let $u, v, \sigma$ be the first singular vectors and value of $R$. Let $U = \mat(u)$ and $V = \mat(v)$. Then
$U$ and $V$ are either both positive definite or negative definite.

Moreover, for any value $\alpha \neq 0$, the pair $(\alpha U, \sigma / \alpha V)$ minimizes $\|L  - X \kron Y\|^2_F$.
\end{theorem}

The proof is a consequence of~\citep[Thm. 11]{Loan93}. This shows that if $L$ is initially positive definite, setting the sign of $\alpha$ based on whether $U$ and $V$ are positive or negative definite\footnote{This can easily be done simply by checking the sign of the first diagonal coefficient of $U$, which will be positive if and only if $U \succ 0$.}, and updating \[\begin{cases}
    L_1 \leftarrow \alpha \, L_1UL_1 \\
    L_2 \leftarrow \sigma/\alpha \, L_2VL_2
  \end{cases}\]
maintains positive definite iterates. Given that if $L_1 \succ 0$ and $L_2 \succ 0$, $L_1 \kron L_2 \succ 0$, a simple induction then shows that by choosing an initial kernel estimate $L \succ 0$, subsequent values of $L$ will remain positive definite. 

By choosing $\alpha$ such that the new estimates $L_1$ and $L_2$ verify $\|L_1\| = \|L_2\|$, we verify all the conditions of Eq.~\ref{eq:4}.

\subsection{Algorithm for joint updates}
\label{app:joint-pseudocode}
Theorem~\ref{thm:joint-psd} leads to a straightforward iteration for learning matrices $L_1$ and $L_2$ based on the decomposition of the Picard estimate as a Kronecker product.

\begin{algorithm}[H]\small
   \caption{\algob iteration}
   \label{algob}
   \begin{algorithmic}
   \STATE {\bfseries Input:} Matrices $L_1, L_2$, training set $T$, step-size $a \geq 1$.
   \FOR{$i=1$ {\bfseries to} maxIter}
   \STATE $U,\sigma,V \leftarrow $ power\_method$(L^{-1} + \Delta)$ to obtain the first singular value and vectors of matrix $R$.
   \STATE $\alpha \leftarrow \text{sgn}(U_{11}) \sqrt{\sigma \|L_2VL_2\| /\|L_1UL_1\|}$
   \STATE $L_1 \leftarrow L_1 + a (\alpha\,L_1UL_1 - L_1)$
   \STATE $L_2 \leftarrow L_2 + a (\sigma/\alpha\, L_2VL_2)$
   \ENDFOR
   \textbf{return} $(L_1,L_2)$
\end{algorithmic}
\end{algorithm}

\end{document}